\documentclass[letterpaper, 10 pt, conference]{ieeeconf}  

\IEEEoverridecommandlockouts                              
\usepackage{cite}
\usepackage{amsmath}
\usepackage{amssymb}
\usepackage{bm}
\usepackage{graphicx}
\usepackage{float}
\let\labelindent\relax
\usepackage{enumitem}
\usepackage[mathscr]{euscript}
\usepackage[version=4]{mhchem}
\usepackage{siunitx}
\usepackage{longtable,tabularx}
\usepackage{caption}
\usepackage{subcaption}
\usepackage[utf8]{inputenc}

\usepackage[colorlinks]{hyperref}

\def\myMediumFigureScale{0.33}
\def\myMSFigureScale{0.24}
\def\mySmallFigureScale{0.16}
\def\myLineScale{1}

\usepackage{setspace}
\usepackage[normalem]{ulem}

\newtheorem{theorem}{Theorem}

\newtheorem{corollary}{Corollary}
\newtheorem{lemma}{Lemma}

\usepackage{algorithm}
\usepackage[noend]{algpseudocode}
\makeatletter
\def\BState{\State\hskip-\ALG@thistlm}
\makeatother
\algnewcommand{\Or}{\textbf{or}\,}
\algnewcommand\algorithmicswitch{\textbf{switch}}
\algnewcommand\algorithmiccase{\textbf{case}}
\algdef{SE}[SWITCH]{Switch}{EndSwitch}[1]{\algorithmicswitch\ #1\ \algorithmicdo}{\algorithmicend\ \algorithmicswitch}%
\algdef{SE}[CASE]{Case}{EndCase}[1]{\algorithmiccase\ #1}{\algorithmicend\ \algorithmiccase}%
\algdef{SE}[ELSE]{Else}{EndElse}[1]{\algorithmicelse\ #1}{\algorithmicend\ \algorithmicelse}%
\algtext*{EndSwitch}%
\algtext*{EndCase}%
\algtext*{EndElse}%

\usepackage[usenames,dvipsnames]{color}
\definecolor{darkblue}{rgb}{0,0,.6}
\definecolor{darkred}{rgb}{.7,0,0}
\definecolor{darkgreen}{rgb}{0,.6,0}
\definecolor{OliveGreen}{cmyk}{0.64,0,0.95,0.40}
\definecolor{CadetBlue}{cmyk}{0.62,0.57,0.23,0}
\definecolor{lightlightgray}{gray}{0.93}

\newcommand{\N}{\mathbb{N}}
\newcommand{\R}{\mathbb{R}}

\newcommand{\set}[1]{\{#1\}}

\newcommand{\argmin}{\mathop{\mathrm{argmin}}} 

\newcommand{\astar}{A$^*\,$}
\newcommand{\coastar}{COA$^*\,$}

\title{\LARGE \bf
A Generalized A$^*$ Algorithm for Finding Globally \\ Optimal Paths in Weighted Colored Graphs
}

\author{Jaein Lim$^{1}$ and Panagiotis Tsiotras$^{2}$
	\thanks{$^{1}$Jaein Lim is a graduate student at the School of Aerospace Engineering, Georgia Institute of Technology, 
		Atlanta. GA, USA. Email:
		jaeinlim126@gatech.edu}%
	\thanks{$^{2}$ Panagiotis Tsiotras is a Professor and David and Andrew Lewis Chair at the School of Aerospace Engineering and Associate Director at the Institute for Robotics and Intelligent Machines, Georgia Institute of Technology, Atlanta. GA, USA. Email: tsiotras@gatech.edu}
}

\begin{document}
	
\maketitle
\thispagestyle{empty}
\pagestyle{empty}

\begin{abstract}
Both geometric and semantic information of the search space are imperative for a good plan. We encode those properties in a weighted colored graph (geometric information in terms of edge weight and semantic information in terms of edge and vertex color), and propose a generalized \astar to find the shortest path among the set of paths with minimal inclusion of low-ranked color edges. 
We prove the completeness and optimality of this \textit{Class-Ordered} \astar (\coastar) algorithm with respect to the hereto defined notion of optimality. 
The utility of \coastar is numerically validated in a ternary graph with feasible, infeasible, and unknown vertices and edges for the cases of a 2D mobile robot, a 3D robotic arm, and a 5D robotic arm with limited sensing capabilities.
We compare the results of \coastar to that of the regular \astar algorithm, the latter of which finds a shortest path regardless of uncertainty, and we show that the \coastar dominates the \astar solution in terms of finding less uncertain paths.
\end{abstract}

\section{Introduction}
Path planning in partially known or unstructured environments is a relatively under-studied problem~\cite{Janson2018}, compared to the vast effort dedicated in past decades to efficiently solve for optimal paths in fully known search spaces~\cite{Dijkstra1959, Hart1968, Bohlin2000, Karaman2011, Arslan2013, Janson2015, Gammell2015, Dellin2016, Mandalika2019, Strub2020}.
In many applications however, the environment is rarely fully known to the planning agent because of either sensor limitations and/or memory constraints. What is a suitable notion of optimality in such a search space, so that the autonomous agent can make reasonable and consistent decisions?
While the notion of optimality is clear and consistent when the search space is known, in the case of uncertain search spaces one has several options. For example, the space can be approximated with a weighted graph, with each edge assigned a cost; then the optimal path is simply the path with minimum cost (e.g., the shortest path) in this graph. Search methods based on this idea have been the backbone to many numerical approaches developed for planning in dynamic or partially known environments~\cite{Koenig2002, Koenig2004, Likhachev2008, Aine2016,Kambhampati1986, Jung2007, Cowlagi2012, Hauer2015}.

A common framework to cope with partially observable or dynamic environments is re-planning~\cite{Koenig2002, Koenig2004, Likhachev2008, Aine2016}. In this framework, obstacles are a priori unknown and are only revealed during the course of the plan execution; a plan is first found on an optimistically perceived environment with no obstacles and then repaired whenever it is found to be infeasible. The benefit of these methods is their algorithmic efficiency to optimally repair the graph based on the newly perceived part of the environment. A fundamental question remains unanswered, however: “what should an optimal policy strive for in an unknown space?”~\cite{Janson2018}.
In the works of~\cite{Kambhampati1986, Jung2007, Cowlagi2012, Hauer2015}, when the search space is only partially known, the “risk” of traversing unknown edges is combined to the original distance of the edge to construct a new cost measure. The optimal path is then defined as the one that minimizes the cost with respect to this new measure. These methods treat planning in partially known search spaces as a shortest path planning problem with respect to a user-defined measure, where uncertain paths are deferred. Yet, traversing through some high-uncertainty edges cannot be explicitly avoided, if such a path is short enough in terms of the user-defined measure. Moreover, the user-defined measure is not consistent, in general, for re-planning purposes, e.g., when new information changes the perceived environment, reusing the previous search results becomes difficult.

To illustrate the issues with potential inconsistencies, consider two agents with identical planning strategies acting in the same environment. 
Assume that one agent has more accurate information about the environment than the other agent. 
If we were to choose one of the two solution paths provided by either agent, then a reasonable choice would be to choose the path computed by the agent having more information, 
regardless of the path length. This is because a reasonable measure on the quality of the path should be that it is monotonic with the information of the underlying planner. 
However, monotonicity cannot be imposed, in general, in a soft-constrained framework. One can easily construct counter-examples where the less knowledgeable agent produces a “shorter” path than the more knowledgeable agent, by modifying how much the uncertainty is penalized. 

To this end, we pay close attention to a more explicit approach presented in~\cite{Aliabdi2019, Wooden2006}, which defines a notion of optimality in the semantic search space. In that approach, the vertices and edges of the graph are classified into equivalent classes (or ``colors") so as to encode semantic information of the perceived environment. Afterwards, a constrained shortest path planning problem is solved in this weighted colored graph.
The optimal path is not necessarily the shortest path, but it is constrained by the constituent colors. In~\cite{Aliabdi2019} the planning problem in an unknown environment is cast as a constrained shortest path planning problem in a weighted bi-colored graph, where the vertices are partitioned into white (known) or gray (unknown) vertices to approximate partially known environments by excluding known, but infeasible vertices. The authors then find the shortest path in the graph such that the number of gray vertices does not exceed a certain number, or the path that has the least number of gray vertices and whose length does not exceed a certain threshold.

A general weighted colored graph is used in~\cite{Wooden2006} to relax the binary representation of the environment to encode various ranked semantic information of the search space. A weight function devised in~\cite{Wooden2006} assigns each edge to a positive real number incorporating the original edge weight and color, such that a total ordering is placed over paths with mixed colors. This total ordering favors paths with a fewer number of edges of inferior classes, and favors shorter paths if the constituting colors are the same. 
With the modified weight, a standard graph search (e.g., Dijkstra \cite{Dijkstra1959}) ensures that the optimal path respects the total order. 
That is, the resulting path is guaranteed to be the shortest path among the set of paths which include the minimum number of inferior class edges. However, global properties of the entire graph, such as the cardinality of the edge set of each color, are required to compute the modified weight, which may be intractable to compute online.

In this work, we extend regular edge-relaxation algorithms (for instance, \astar \cite{Hart1968}) to find an optimal path in a weighted colored graph based on the notion of optimality defined in~\cite{Wooden2006}. We will adopt the term ``class" over ``color" as in~\cite{Wooden2006} to differentiate from the graph-theoretic coloring problem, since the color of a vertex is not dictated by its adjacency, i.e., there may be two adjacent vertices with the same color. 
Informally, an optimal path is the shortest path in the set of paths with minimal inclusion of inferior-class edges. 
The original weight or the length of the path is used only as a tie-breaker among paths within the same class. For example, a shorter path with bad edges will be considered worse than a longer path with better constitute edges. 
This total ordering over the set of paths will be formally defined in the next section. The aim of our work is to find an optimal path without resorting to the global properties of the graph by propagating forward the cost-to-come and the path-class information to build an optimal search tree. The key insight is to use an abstract priority queue to order the expansion of optimal path candidates. We call this algorithm \textit{Class-Ordered} \astar (\coastar).

One immediate application of planning on a colored graph is that we can plan for an information-conservative path in a partially known environment. 
Suppose we  partition the search environment into three classes: feasible, infeasible, unknown. Then, we can find the shortest path that maximizes feasibility (i.e., minimize infeasible and unknown vertices). For example, if there exist two paths, an unknown short path and a feasible longer path, then the feasible path is ranked higher over the short, unknown path. 
Consequently, the agent traverses through the region that is already known. This may be useful for planning problems such as a surgical robot, a marine probe, or a rescue robot where the impact on the surroundings should be minimized. 

Another application where planning on a colored graph is beneficial is the case of a cooperative multi-agent planning problem, where agents having only local information communicate with each other by exchanging their local plans in order to to achieve a consensus on a globally acceptable optimal solution~\cite{Torreno2014, Nissim2014, Lim2020}.
We can encode information passed to other agents with a ``color" depending on the information quality (e.g., agent credibility, noise-level, etc.,) such that we can suppress the communication of inferior information below a certain class. 
In a soft-constrained framework, each agent must {receive and merge}
the transmitted information fully before determining whether to use it or discard it, and the agents will use any ``shorter" path regardless of the quality of the path. 
The colored graph planning provides another layer of ordering of paths (besides a single length+risk measure, for instance), and hence it is capable of rejecting shorter paths of inferior class. 

The contributions of this paper can be summarized as follows. First, we develop a complete and optimal algorithm to solve for the shortest path among the set of paths with minimal inclusion of low-class edges in a weighted colored graph. 
Second, we prove the completeness and optimality of the proposed algorithm. 
The proposed algorithm efficiently builds an optimal search tree by heuristically and lazily expanding only the optimal path candidates. 
Third, we validate the utility of planning on a weighted colored graph using numerical experiments. 


\section{Problem Formulation}\label{sec2}

Before discussing the problem to be solved in this paper, let us define some background notation that will be frequently used in the rest of the  paper. 
\subsection{Weighted Colored Graph}
Let $G =(V,E)$ be a graph with vertices $V$ and edges $E.$ Let $\phi_V : V \to \mathcal{K}$ be a perception vertex function that classifies each vertex $v\in V$ to a class $k \in \mathcal{K} = \set{1,...,K}$, such that $\phi_V$ partitions the set $V$, that is, $V = \bigcup_{k\in \mathcal{K}}V_k$ where $V_k = \set{v\in V : \phi_V(v) = k}$ and $V_i \cap V_j = \varnothing$ for $i\neq j$.
Similarly, define $\phi_E : E\to \mathcal{L}$ to be a perception edge function that classifies each edge $e\in E$ to a class in $\mathcal{L} = \set{1,...,L},$ such that $E = \bigcup_{\ell\in \mathcal{L}} E_\ell$ where each $E_\ell = \set{e\in E : \phi_E(e) = \ell}$ and $E_i \cap E_j = \varnothing$ for $i\neq j$. 
We will assume that the edge class set is larger than the vertex class set, i.e., $L\geq K$, and that, for each edge, $e=(u,v)$, it holds $\phi_E(e) \geq \max \set{\phi_V(u),\phi_V(v)}$. 
This assumption allows us to quickly underestimate the edge class by classifying the end vertices first. 
This serves as a heuristic edge classification before the actual, and perhaps expensive, edge classification takes place. 
Also, for each edge $e\in E$, a weight function $c: E\to \R_{+}$ assigns a non-negative real number, e.g., the distance to traverse this edge.
We will assume that the functions $\phi_E$ and $c$ are given, but are expensive to compute. Hence, we will approximate them with functions $\hat{\phi}_E$, and $\hat{c}$, respectively, that underestimate their true values, i.e., $\hat{\phi}_E \leq \phi_E$ and $\hat{c} \leq c$. We will use these admissible heuristic estimators to prioritize the expansion of optimal path candidates, in an attempt to delay the actual computation until it becomes necessary.  

\subsection{Optimal Path}

Define a path $\pi = (v_1, v_2 ,\ldots, v_m)$ on the graph $G=(V,E)$ as an ordered set of distinct vertices $v_i \in V$, $i = 1,\ldots,m$ such that for any two consecutive vertices $v_i, v_{i+1}$, there exists an edge $e = (v_i, v_{i+1}) \in E.$ 
Throughout this paper, we will interchangeably denote a path as the set of such edges.
Let $v_s, v_g \in V$ be the start and goal vertices, respectively. Denote by $\Pi(v_s,v_g)$, or $\Pi$ for short if there is no danger for ambiguity, the set of all acyclic paths connecting $v_s$ and $v_g$ in $G$.
Let $\Pi_k$ be the set of paths in which the worst (greatest) edge class included in each path in $\Pi_k$ is exactly $k$. That is, 
\[
\Pi_k = \set{\pi \in \Pi : N(\pi,k)>0 \;\mathrm{ and }\; N(\pi, \ell) =0, \forall \ell >k},
\] 
where $N(\pi,k) = |\set{e \in \pi : \phi_E(e) = k} |$ is the number of edges that are of class $k$ in path $\pi$.
Furthermore, define
\[
\Pi_k^i = \set{\pi \in \Pi_k : N(\pi, k) =i},
\]
to be the set of class-$k$ paths which have exactly $i$ edges of class $k$. We shall impose a total order on the set of paths with $\Pi_k^i \prec \Pi_\ell^j$ for any $i,j$ whenever $k< \ell$ and $\Pi_k^i \prec \Pi_k^j$ for all $i<j$. 
Hence, we define the optimal path set $\Pi^*$ as the nonempty set of paths having best worst-class edge with the least number of worst-class edges, that is, 
\[
\Pi^* = \min_{i \in \N}\min_{k \in \mathcal{K}} \set{\Pi_k^i},
\] where the minimum is defined with respect to the total path order.
We wish to solve for the optimal path $\pi^*$ which is the shortest path in $\Pi^*$, that is, 
\[
\pi^* = \argmin_{\pi \in \Pi^*} \sum_{e\in \pi} c(e).
\]

\subsection{Search Tree}

Initially, all vertices and edges of $G$ are implicitly defined and are unclassified. As the search propagates, an explicit search tree $T$ rooted at $v_s$ is built by adding classified vertices and edges.
Each vertex $v$ in the tree $T$ has the following values explicitly assigned. 
\begin{itemize}
	\item \textsf{parent}: parent vertex. 
	\item \textsf{cost-to-come}: cost-to-come accumulated along the current best path in the tree, denoted with $g_T(v).$ 
	\item \textsf{path-class}: number of edges of each class in the current best path, denoted with $\theta_T (v).$
\end{itemize}

The range of $g_T$ is the set of non-negative real numbers, whereas the range of $\theta_T$ is the set of sequences of integers of finite length $|\mathcal{K}|$, i.e., $\theta_T(v) = (N(\pi,k))_{k\in \mathcal{K}}$, where the number of edges that are of class $k$ in the current path $\pi=(v_s, \ldots, v)$ is stored in the $k$-th index. 
By the total order defined in the previous section, any two paths $\pi_1=(u_1,\ldots,u_m)$ and $\pi_2=(v_1,\ldots,v_n)$ with the same first and end vertices (i.e., $u_1=v_1$ and $u_m=v_n$) can be compared. For example, if there exists another tree $T'$ that has a better path from $v_s$ to $v$, then we have $\theta_{T'}(v) \prec \theta_T(v)$. 
We will denote the sequence of integers $(N(\pi,k))_{k\in \mathcal{K}}$ for a path $\pi = (v,\ldots,v')$ with $\theta(v,v')$. 
Then $\theta$ is sub-additive along the path, $\theta(v,v'') \preceq \theta(v,v') + \theta (v',v'')$. 
We also define an admissible estimate of $\theta(v,v')$ with $\hat{\theta}(v,v'),$ such that $\hat{\theta} \preceq \theta.$
A trivial heuristic estimate of $\hat{\theta}$ is the zero sequence. 

Note that the standard \astar algorithm stores for each vertex in a search tree $T$ a back-pointer to its parent vertex and also stores the cost-to-come accumulated by traversing the back-pointed path. Despite the fact that the cost-to-come of a vertex in $T$ can be retrieved by following the back-pointed path in the forward direction, the numerical value of the cost-to-come is explicitly stored at each vertex to expedite the ordering of expanding leaf vertices so that  the completeness and optimality of the algorithm with respect to the cost-to-come are preserved. 
The same idea can be generalized to a total ordering of vertices, based on an order defined over a set of paths, without losing correctness of the algorithm. The completeness and optimality of such generalization with respect to any ordered set of paths can be also proven, as it will be shown formally in Section~\ref{sec4}. 

\subsection{Priority Queue}

We will use an edge queue $Q$ to prioritize the expansion (i.e., evaluation) of edges based on an admissible estimate of the path class, breaking ties using conventional path cost estimates.
For each edge $e=(u,v)$ in $Q$, we associate four keys: $k_1(e)=\theta_T(u)+\hat{\theta}(u,v)+\hat{\theta}(v,v_g)$, $k_2(e)=\theta_T(u)+\hat{\theta}(u,v)$, $k_3(e)= g_T(u)+\hat{c}(u,v) + \hat{h}(v),$ and $k_4(e)= g_T(u)+\hat{c}(u,v) $ to prioritize edges in lexicographical order.
Note that $\hat{\theta} \preceq \theta$, $\hat{c}\leq c$, and $\hat{h}\leq h$ are admissible estimates of the path class, edge cost, and cost-to-go functions respectively, and $\theta_T$ and $g_T$ are the path class and the cost-to-come of the best path found so far by the algorithm.
These heuristic estimates of the path class and the path cost delay the actual evaluation of the edge, as the priority queue prioritizes edges based on their heuristic estimates.

\section{Algorithm}\label{sec3}

We begin the search by initiating a search tree $T$ with the start vertex as the root and by putting the promising outgoing edges in the priority queue $Q$, where a promising edge is an edge such that the path utilizing this edge may improve the current search tree $T$.
That is, only edges that could potentially improve the current solution with heuristic estimates are first chosen, and then among the chosen edges, only those that can potentially improve the current path to the child vertex of the edge with the heuristic estimates are selected to be inserted in the queue.
The edge with the highest index is removed from $Q$ for evaluation, and then the edge is added to the tree $T$ if such an evaluation reveals that the edge either improves the child vertex's \textsf{path-class} or reduces the child vertex's \textsf{cost-to-come} value within the same class.
The outgoing edges of the newly added leaf vertex of $T$ are then inserted in the queue accordingly.
The iteration continues until the goal vertex is reached. The resulting search tree consists of connected classified edges rooted at $v_s$, and it contains an optimal path from $v_s$ to $v_g$. The solution path is traced back from $v_g$ to its $\textsf{parent}$ up to $v_s$.

\begin{algorithm}
	\caption{\textsf{Class Ordered A}$^*(v_s, v_g, G)$}\label{a:COAS}
	\begin{algorithmic}[1]
		\State $T \gets \varnothing$, $Q \gets \varnothing$
		\State $v_s.\textsf{parent}=\varnothing$
		\State $v_s.\textsf{cost-to-come} = 0, v_s.\textsf{path-class} =0$ 
		\State \textsf{enqueueOutgoingEdges}($v_s, Q$)
		\While{true}
		\State $(u,v) \gets $ \textsf{popFrontEdge}($Q$)
		\If {$u=v_g$} \label{a:terminal_cond}
		\State \textbf{break}
		\EndIf
		\If {$\theta_T(u)+\theta(u,v) \prec \theta_T(v)$} 
		\State \textsf{rewireTree}($(u,v), T$)
		\State \textsf{enqueueOutgoingEdges}($v, Q$)
		\ElsIf{$\theta_T(u)+\theta(u,v) = \theta_T(v)$}
		\If {$g_T(u)+c(u,v) < g_T(v)$ }
		\State \textsf{rewireTree}($(u,v), T$)
		\State \textsf{enqueueOutgoingEdges}($v,Q$)
		\EndIf
		\EndIf
		\EndWhile
	\end{algorithmic}
\end{algorithm}

\begin{algorithm}
	\caption{\textsf{enqueueOutgoingEdges A}$(v, Q)$}\label{a:enqueue}
	\begin{algorithmic}[1]
		\ForAll{$w \in \textsf{neighbors}(v)$}
			\If {$\theta_T(v) + \hat{\theta}(v,w) + \hat{\theta}(w,v_g)  \preceq \theta_T(v_g)$} 
				\If {$\theta_T(v) + \hat{\theta}(v,w) \prec \theta_T(w)$}
					\State $Q \gets Q \cup (v,w)$
				\ElsIf {$\theta_T(v) + \hat{\theta}(v,w) = \theta_T(w)$}
					\If {$g_T(v)+\hat{c}(v,w) + \hat{h}(w) \leq g_T(v_g)$}
						\If{$g_T(v)+\hat{c}(v,w) < g_T(w)$}
							\State $Q \gets Q \cup (v,w)$
						\EndIf
					\EndIf
				\EndIf
			\EndIf
		\EndFor
	\end{algorithmic}
\end{algorithm}

\begin{algorithm}
	\caption{\textsf{rewireTree}$((u,v),T)$}\label{a:rewire}
	\begin{algorithmic}[1]
		\State $v.\textsf{parent} \gets u$ 
		\State $v.\textsf{cost-to-come} \gets g_T(u)+c(u,v)$
		\State $v.\textsf{path-class} \gets \theta_T(u) + \theta(u,v)$
		\State $T \gets T \cup (u,v)$
	\end{algorithmic}
\end{algorithm}

\begin{figure}[thpb]
	\centering
	\begin{subfigure}{\myMSFigureScale\textwidth}
		\includegraphics[width=\myLineScale\linewidth]{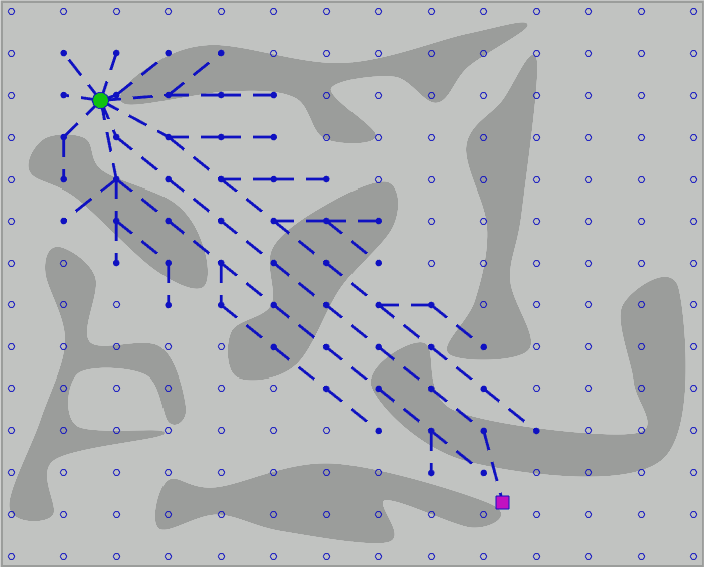}
		\caption{}
	\end{subfigure}\hfill
	\begin{subfigure}{\myMSFigureScale\textwidth}
		\includegraphics[width=\myLineScale\linewidth]{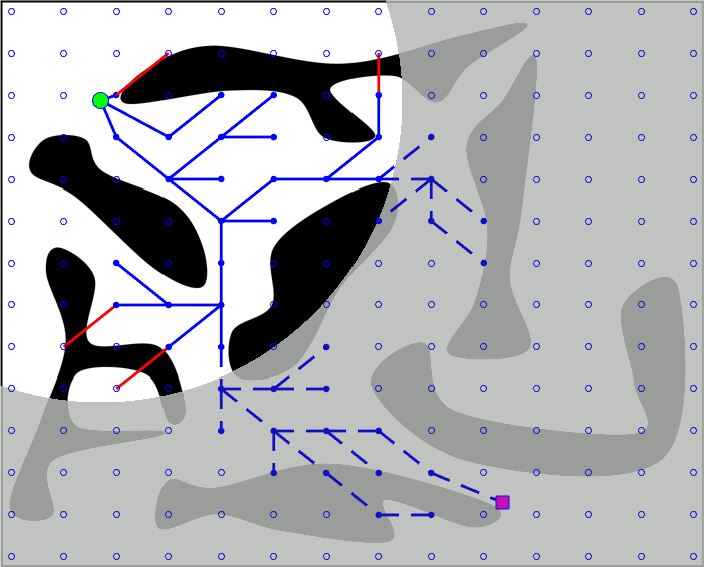}
		\caption{}
	\end{subfigure}\hfill
	\begin{subfigure}{\myMSFigureScale\textwidth}
		\includegraphics[width=\myLineScale\linewidth]{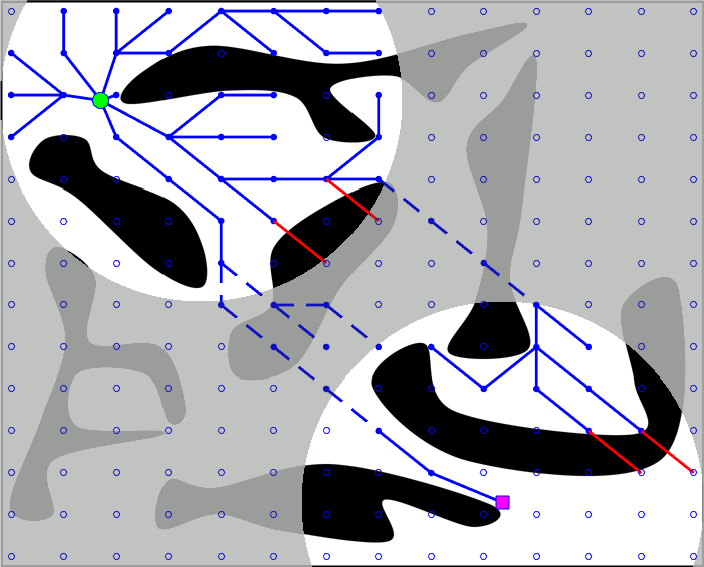}
		\caption{}
    \end{subfigure}\hfill
	\begin{subfigure}{\myMSFigureScale\textwidth}
		\includegraphics[width=\myLineScale\linewidth]{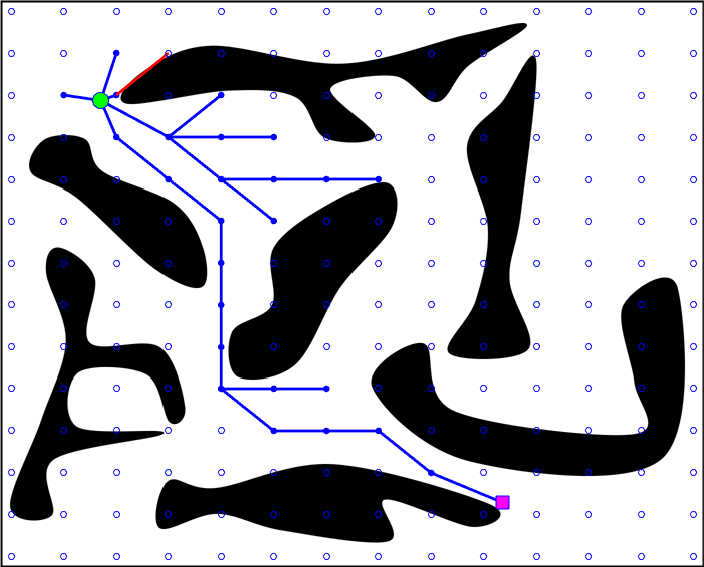}
		\caption{}
	\end{subfigure}
	\caption{Search trees built from top left start toward bottom right goal on different perceived environments, where gray region is unknown: 		
	a) completely unknown environment; 
	b) partially known around start; 
	c) partially known around start and goal;
	d) fully known environment. Solid blue line, dashed blue line, and solid red line indicate three different edge classes: feasible, unknown, infeasible, respectively, classified along the expansion of \coastar.
	}
	\label{f:result_example}
\end{figure}

\section{Analysis}\label{sec4}

We prove the completeness and optimality of \coastar by extending the results of \cite{Hart1968}. The key observation here is that the original \astar algorithm uses an ordinary ordering of real numbers for expanding paths to find a path minimizing the cost-to-come, but the vertex ordering of the algorithm can be generalized to a total ordering of paths. In our analysis, we will assume that a total path ordering exists such that for any pair of paths having the same start and goal vertices, the order gives a binary relation between the two. Then we will prove that the resulting path of such an algorithm is optimal with respect to the defined ordering. 

We will present an edge queue version of the proof instead of the vertex queue proof of the original \astar to utilize admissible estimators of edge evaluation functions. This choice relaxes the heuristic look-ahead over to an edge, such that the actual evaluation of the edge is delayed until it is necessary, under the assumption that edge evaluations are computationally expensive \cite{Cohen2014, Gammell2015, Strub2020}. However, since the goal is to find a path from a start vertex to a goal vertex, we need to first define the trivial edges involving the start and goal vertices. To this end, we define a start edge $s=(v_s, v_s)$ for a start vertex $v_s$, and a goal edge $t=(v_g,v_g)$ for a goal vertex $v_g$, where the evaluations of these trivial edges are naturally defined as: $c(s)=0$, $c(t)=0$, $\phi_E(s) = \phi_V(v_s)$ and $\phi_E(t) = \phi_V(v_g).$ 

With some abuse of notation, we will denote the optimal path $(e,\ldots,e')$ from a start edge $e$ to an end edge $e'$ in the graph with $\pi(e,e')$. 
Also, we will denote the path from an edge $e$ to $e'$ found by the \coastar algorithm so far with $\pi_T(e,e'),$ where $T$ represents the current search tree. 
We will assume that there exists an ordering for any pair of paths from the same start and end edges.  
For example, we have $\pi(e,e') \preceq \pi_T(e,e')$ since any search tree built by the algorithm cannot contain a strictly better path than the optimal path. 
Also, we denote an {underestimating heuristic} path from $e$ to $e'$ with $\hat{\pi} (e,e')$, such that $\hat{\pi}(e,e') \preceq \pi(e,e')$ for any edges $e, e'$.
Concatenation of paths is denoted with $\pi(e, e') \cup \pi(e', e'') = \pi(e, e'')$ for two adjacent paths. 
If $e'$ and $e''$ are adjacent edges on the optimal path, then, $\pi(e, e') \cup e'' = \pi(e,e'').$
We will use $\pi(e,e)=\varnothing$ for any $e \in E$, and $\varnothing \preceq \pi(e,e')$ for any edges $e, e'$.

We are now ready to present Lemma~\ref{lemma:concatenation_lemma}, which states that the path ordering is invariant under the concatenation with the same subpath.

\begin{lemma}\label{lemma:concatenation_lemma}
	Let $\hat{\pi}(e,t) \preceq \pi(e,t)$. 
	Then $\pi(s,e)\cup \hat{\pi}(e,t) \preceq \pi(s,e)\cup \pi(e,t)$, for any $e$.
\end{lemma}
\begin{proof}
	Fix an edge $e=(u,v)$, and assume that $\hat{\pi}(e,t) \preceq \pi(e,t).$ There are two cases to consider. 
	\begin{enumerate}[leftmargin=*]
		\item[a)] \textit{$\hat{\theta}(u,v)+\hat{\theta}(v,t) \prec \theta(u,v)+\theta(v,t).$}
			Then, $\theta(s,u)+\hat{\theta}(u,v)+\hat{\theta}(v,t) \prec \theta(s,u)+\theta(u,v)+\theta(v,t)$, since the ordering is invariant under the addition. 
		\item[b)] \textit{$\hat{\theta}(u,v)+\hat{\theta}(v,t) = \theta(u,v)+\theta(v,t)$} and {$\hat{c}(u,v) + \hat{h}(v) \leq c(u,v)+h(v).$ }
			Then, $\theta(s,u)+\hat{\theta}(u,v)+\hat{\theta}(v,t) = \theta(s,u)+\theta(u,v)+\theta(v,t)$ and $g(u)+\hat{c}(u,v)+\hat{h}(v) \leq g(u)+c(u,v)+h(v)$.
	\end{enumerate}
	In either case, we have $\pi(s,e)\cup \hat{\pi}(e,t) \preceq \pi(s,e)\cup \pi(e,t),$ thus completing the proof. 
\end{proof}

Next, Lemma~\ref{lemma: baby lemma} states that if an edge on an optimal path is not evaluated, then the edge queue $Q$ must contain an edge such that an optimal subpath from the start edge to this edge is in the search tree $T.$
\begin{lemma}
	\label{lemma: baby lemma}
	For any unevaluated edge $e$, and for any optimal path $P$ from $s$ to $e$, there exists an edge $e' \in Q$ on $P$ such that $\pi(s,e') = \pi_T(s,e').$
\end{lemma}

\begin{figure}[ht]
	\includegraphics[width=1\linewidth]{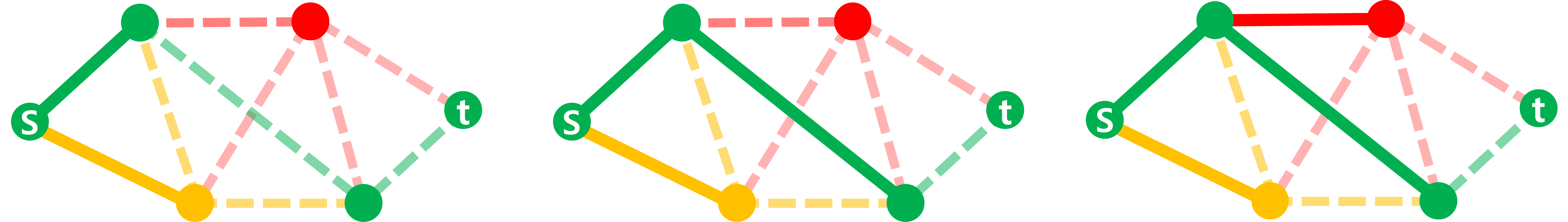}\centering
	\caption{Evolution of search tree $T$ from left to right in a ternary graph with green, yellow, and red classes. \coastar builds the optimal tree $T$ by revealing the classes of the frontier edges in the order defined by the priority queue.}
	\label{f:lemma_fig}
\end{figure}

\begin{proof}
	Let $P = (s=e_0, e_1 ,\ldots , e_n =e).$ 
	If $s \in Q$, that is, if \coastar has not completed the first iteration, let $e' =s.$ Then $\pi(s,s)=\pi_T(s,s)=\varnothing$, and hence the lemma is trivially true. 
	Now, suppose $s$ is evaluated. 
	Let $\Delta$ be the set of all evaluated edges $e_i$ in $P$ such that $\pi(s,e_i) = \pi_T(s,e_i).$
	Then, $\Delta$ is not empty, since by assumption $s\in \Delta.$ 
	Let $e^*$ be the element of $\Delta$ with the highest index. 
	Clearly, $e^* \neq e,$ since $e$ has not yet been evaluated. 
	Let $e'$ be the successor of $e^*$ on $P$.
	Then, $\pi_T(s,e') \preceq \pi_T(s,e^*) \cup e'.$
	However, $\pi_T(s,e^*) = \pi(s,e^*)$ since $e^* \in \Delta.$ 
	Moreover, $\pi(s,e') = \pi(s,e^*) \cup e'$, since $e^*$ and $e'$ are on $P.$
	Therefore $\pi_T(s,e') \preceq \pi(s,e').$
	In general, $\pi_T(s,e') \succeq \pi(s,e')$, and hence $\pi_T(s,e') = \pi(s,e').$ Also, $e'$ must be inserted in $Q$ since $e'$ is adjacent to $e^*$, the highest index edge in $\Delta.$
\end{proof}

\begin{corollary}  		\label{coro: baby corollaray}
	Suppose $\hat{\pi}(e,t) \preceq \pi(e,t)$ for all $e$, and suppose \coastar has not terminated. Then, for any optimal path $P$ from $s$ to a goal $t$, there exists an edge $e' \in Q$ on $P$ with $\pi_T(s,e') \cup \hat{\pi}(e',t) \preceq \pi(s,t).$	
\end{corollary}

\begin{proof}
	By the Lemma~\ref{lemma: baby lemma}, there exists an edge $e'\in Q$ in $P$ with $\pi(s,e')=\pi_T(s,e').$ 
	Then, by applying the Lemma~\ref{lemma:concatenation_lemma},
	\[
	\begin{aligned}
	\pi_T(s,e') \cup \hat{\pi}(e',t) &= \pi(s,e') \cup \hat{\pi}(e',t) \\
	& \preceq \pi(s,e') \cup \pi(e',t) \\
	& = \pi(s,t),
	\end{aligned}
	\]
	where the last equality holds since $e'\in P$.
\end{proof}

Now we are ready to prove the completeness and optimality properties of COA$^*.$
\begin{theorem}
	Suppose $\hat{\pi}(e,t) \preceq \pi(e,t)$ for all $e$. Then \coastar terminates in a finite number of iterations and finds the optimal path from $s$ to a goal $t$, if one exists.
\end{theorem}
\begin{proof}
	We prove this theorem by contradiction. Suppose the algorithm does not terminate by finding an optimal solution. There are three cases to consider: 
	\begin{enumerate}[leftmargin=*]
		\item \textit{The algorithm terminates at a non-goal. } 
		This contradicts the termination condition (Line~\ref{a:terminal_cond}, Algorithm~\ref{a:COAS}).
		
		\item \textit{The algorithm fails to terminate.}
		Let $\pi(s,t)$ be an optimal path from $s$ to the goal $t$. Clearly, no edges with $\pi_T(s,e)\cup \hat{\pi}(e,t) \succ \pi(s,t)$ will ever be evaluated, since by Corollary~\ref{coro: baby corollaray} there is some $e'$ with $\pi_T(s,e')\cup \hat{\pi}(e',t) \preceq \pi(s,t)$, and hence, \coastar will select $e'$ instead of $e$.  
		Since there is only a finite number of acyclic paths from the start edge to any edge, all edges with $\pi(s,e)\cup \hat{\pi}(e,t) \preceq \pi(s,t)$ can be re-evaluated at most a finite number of times. 
		Hence, the only possibility left for \coastar that fails to terminate is when the queue $Q$ becomes empty before a goal is reached. Suppose, ad absurdum, that $Q$ becomes empty before the goal is reached. 
		If there exists a path to the goal, and the goal is not already in the tree $T,$ then there exists at least one edge that could improve the current solution path. 
		Then, this edge would have been inserted in the queue $Q$ by Algorithm~\ref{a:enqueue}, contradicting the assumption that $Q$ becomes empty. 
		Hence, \coastar must terminate. 
		
		\item \textit{The algorithm terminates at the goal without finding an optimal solution.}
		Suppose \coastar terminates at some goal edge $t$ with $\pi_T(s,t) \succ \pi(s,t).$ By Corollary~\ref{coro: baby corollaray}, just before termination, there must exist an edge such that $e'\in Q$ with $\pi_T(s,e')\cup \hat{\pi}(e',t)\preceq \pi(s,t).$ Hence, at this stage, $e'$ would have been selected for evaluation rather than $t$, contradicting the assumption that \coastar terminated.
	\end{enumerate}
	Therefore, the algorithm must terminate by finding the optimal solution, if one exists. 
\end{proof}

Note that we proved the completeness and optimality with respect to the total order defined in Section~\ref{sec2}. 
However, the same technique can be applied to a different total order using more general properties of paths if the properties are additive along the path and the start vertex is already on the optimal path. An immediate corollary of this observation is that \coastar can also find a shortest path among the set of paths with shortest inferior class edges. This is useful, especially for computing heuristics in high-dimensional search spaces, as the distance between two vertices is easier to obtain than counting the possible number of edges in the path connecting the two vertices.

\section{Numerical Results}\label{sec5}
\begin{figure*}[ht]
	\begin{subfigure}{\myMediumFigureScale\textwidth}
		\includegraphics[width=\myLineScale\linewidth]{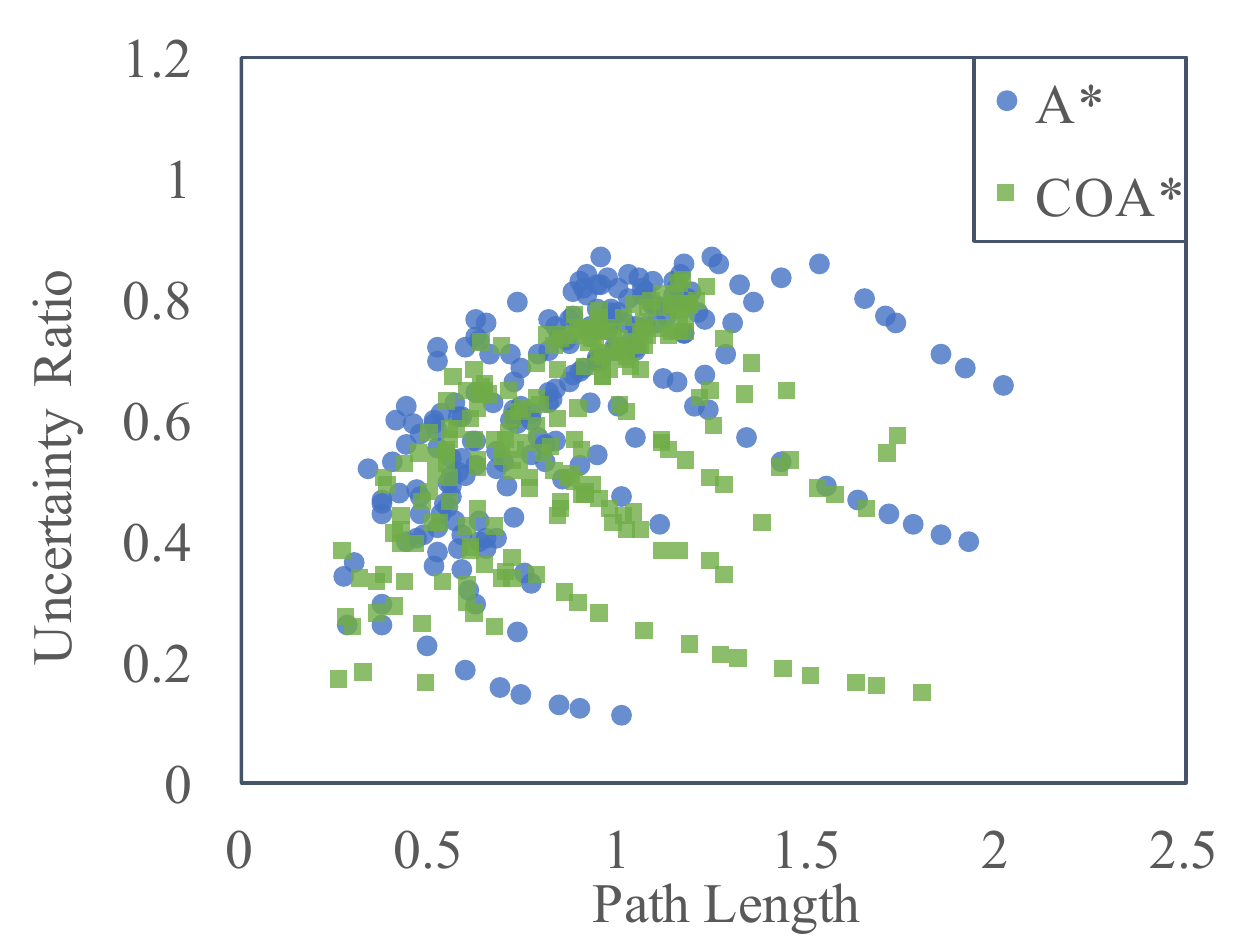}
	\end{subfigure}\hfill
	\begin{subfigure}{\myMediumFigureScale\textwidth}
		\includegraphics[width=\myLineScale\linewidth]{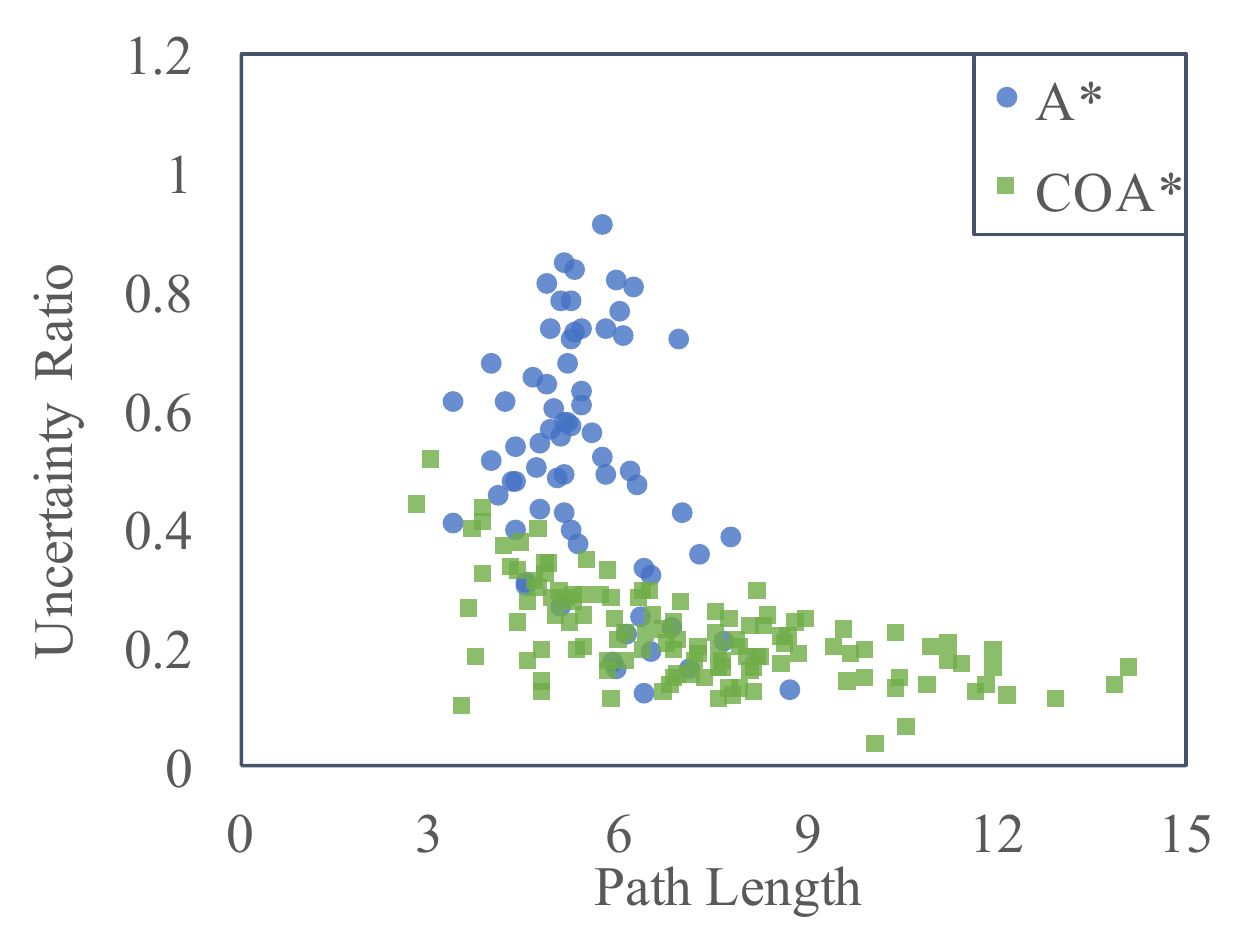}
	\end{subfigure}\hfill
	\begin{subfigure}{\myMediumFigureScale\textwidth}
		\includegraphics[width=\myLineScale\linewidth]{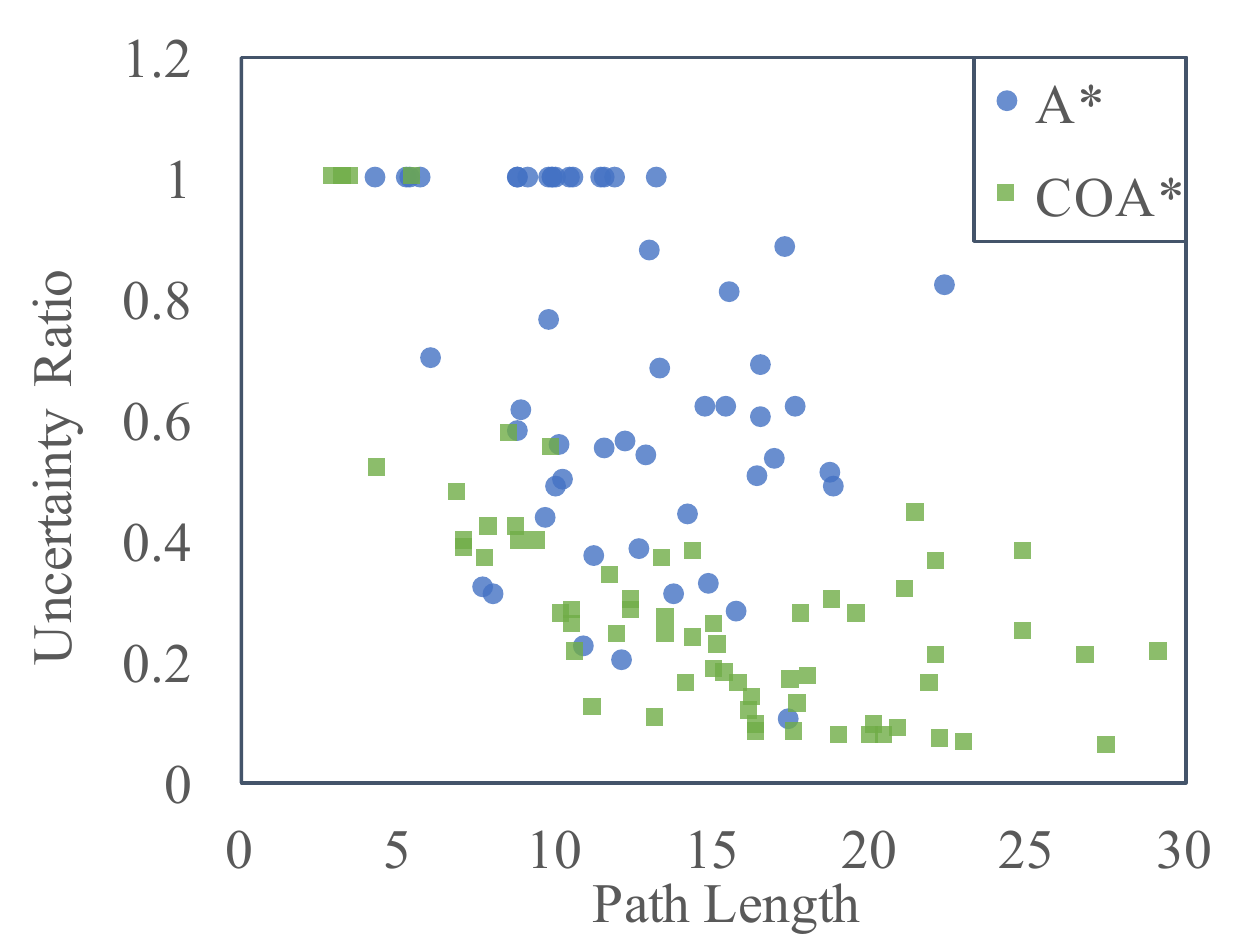}
	\end{subfigure}\hfill
	\caption{Collection of path length and uncertainty ratio for each local path recorded during the perception-plan-action loop, excluding fully known paths: from left to right: 2D mobile robot, 3D robotic arm, 5D robotic arm}
	\label{f:dots}
\end{figure*}

\begin{figure*}[ht]
	\centering
	\begin{subfigure}{\mySmallFigureScale\textwidth}
		\includegraphics[width=\myLineScale\linewidth]{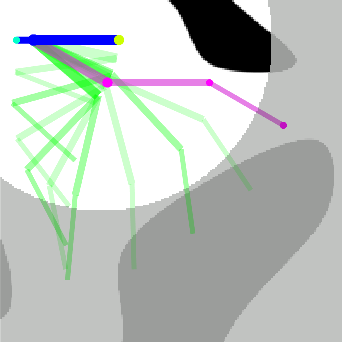}
	\end{subfigure}\hfill
	\begin{subfigure}{\mySmallFigureScale\textwidth}
		\includegraphics[width=\myLineScale\linewidth]{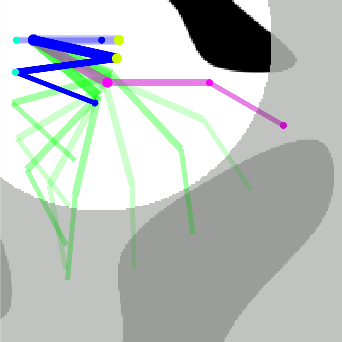}
	\end{subfigure}\hfill
	\begin{subfigure}{\mySmallFigureScale\textwidth}
		\includegraphics[width=\myLineScale\linewidth]{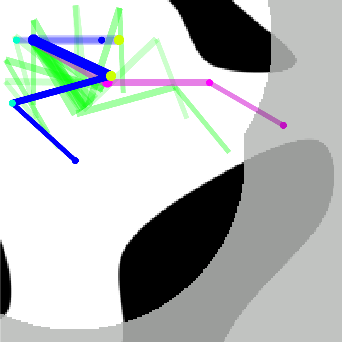}
	\end{subfigure}\hfill
	\begin{subfigure}{\mySmallFigureScale\textwidth}
		\includegraphics[width=\myLineScale\linewidth]{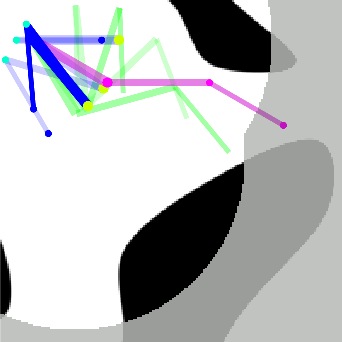}
	\end{subfigure}\hfill
	\begin{subfigure}{\mySmallFigureScale\textwidth}
		\includegraphics[width=\myLineScale\linewidth]{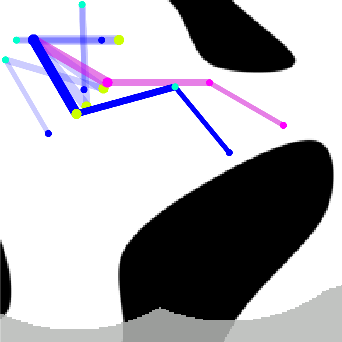}
	\end{subfigure}\hfill
	\begin{subfigure}{\mySmallFigureScale\textwidth}
		\includegraphics[width=\myLineScale\linewidth]{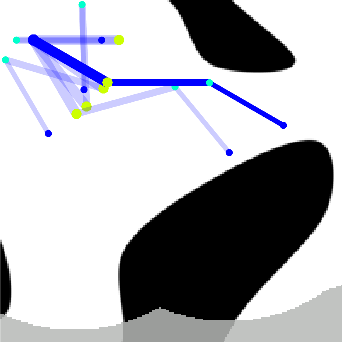}
	\end{subfigure}
	
	\begin{subfigure}{\mySmallFigureScale\textwidth}
		\includegraphics[width=\myLineScale\linewidth]{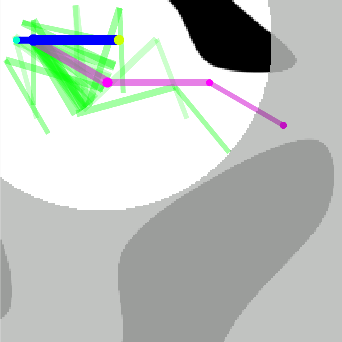}
	\end{subfigure}\hfill
	\begin{subfigure}{\mySmallFigureScale\textwidth}
		\includegraphics[width=\myLineScale\linewidth]{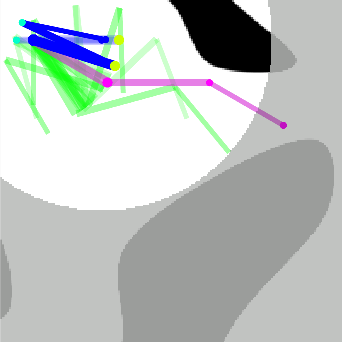}
	\end{subfigure}\hfill
	\begin{subfigure}{\mySmallFigureScale\textwidth}
		\includegraphics[width=\myLineScale\linewidth]{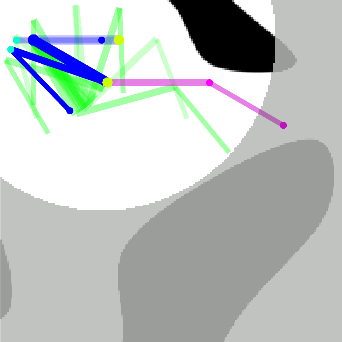}
	\end{subfigure}\hfill
	\begin{subfigure}{\mySmallFigureScale\textwidth}
		\includegraphics[width=\myLineScale\linewidth]{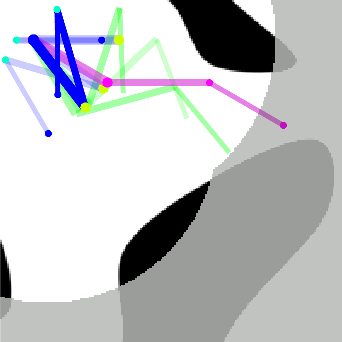}
	\end{subfigure}\hfill
	\begin{subfigure}{\mySmallFigureScale\textwidth}
		\includegraphics[width=\myLineScale\linewidth]{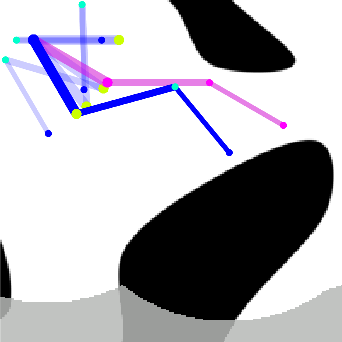}
	\end{subfigure}\hfill
	\begin{subfigure}{\mySmallFigureScale\textwidth}
		\includegraphics[width=\myLineScale\linewidth]{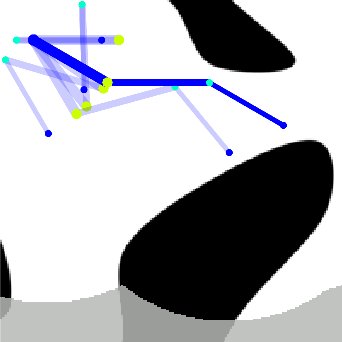}
	\end{subfigure}
	\caption{Propagation of agent with \astar (top row) and \coastar (bottom row) from left to right, where the solid blue arm is the current vertex, the green arms are the vertices on the locally optimal path, the magenta arm is the goal vertex, and the transparent blue arms show the trace of the global solution.}
	\label{f:arms}
\end{figure*}

In this section, we present numerical results of \coastar in scenarios where the planning agent has limited perception of the search space and incrementally gains knowledge of the space as it moves along the local plan. 
We then compare the results of \coastar to a regular \astar which finds a shortest path regardless of uncertainty on an optimistically perceived environment, namely, assuming the unknown region is free of obstacles. 
The perception-plan-action loop is iterated until the agent reaches the goal on a graph with uniformly distributed vertices. At each planning instance, the edges of the graph are classified into three classes: feasible, unknown, and infeasible. 
A sensor with a radial field of view is attached to the moving agent, and all the vertices outside the sensing region are classified as unknown, while each vertex within the sensing region is classified either as feasible or infeasible depending on the corresponding configuration~\cite{Zelinsky1992,Choset1994}. 
Once the agent finds a locally optimal solution in the current map, the agent executes the first action of the plan, i.e., it moves along the first vertex of the locally optimal path, and then updates its map accordingly, as shown in Figure~\ref{f:arms}.  

During the perception-plan-action loop, two measures of interest for each locally optimal plan were recorded at each planning instance, namely, the path length (distance) and the uncertainty ratio. 
The local uncertainty ratio is the proportion of the sum of the length of the uncertain edges to the length of the sum of all the edges in a path, which measures the percentage of the uncertain segment for the path. 
Using sensors with different ranges, about 100 paths were collected for each robot in several search spaces. Figure~\ref{f:dots} shows the results of \astar in comparison to \coastar  for a 2D mobile robot, a 3D robotic arm, and a 5D robotic arm. 
On average, \coastar finds less uncertain paths compared to \astar. This becomes more evident in the higher dimensional search spaces.

\section{Conclusion}

We have introduced the notion of optimality on a weighted colored graph, which encodes both geometric and semantic information of the search space. We present a new search algorithm, the \textit{Class-Ordered} \astar(\coastar) to find a globally optimal path in a weighted colored graph by incrementally building an optimal search tree using a heuristic, and we proved the completeness and optimality of the algorithm. The optimal path of \coastar is the shortest path among the set of paths with minimal inferior class edges. 
In addition, \coastar monotonically finds a better path when the underlying graph has a strictly better class of vertices and edges. 
Finally, \coastar was numerically evaluated on a ternary graph with feasible, unknown, and infeasible classes against the standard \astar algorithm, which finds a shortest path regardless of uncertainty. 
The results of these numerical experiments confirmed the superiority of \coastar in terms of finding safer plans when the search space is partially known. 

\section*{Acknowledgement}
This work has been supported by ARL under DCIST CRA W911NF-17-2-0181 and NSF under award IIS-2008686.


\begin{thebibliography}{10}
	\providecommand{\url}[1]{#1}
	\csname url@rmstyle\endcsname
	\providecommand{\newblock}{\relax}
	\providecommand{\bibinfo}[2]{#2}
	\providecommand\BIBentrySTDinterwordspacing{\spaceskip=0pt\relax}
	\providecommand\BIBentryALTinterwordstretchfactor{4}
	\providecommand\BIBentryALTinterwordspacing{\spaceskip=\fontdimen2\font plus
		\BIBentryALTinterwordstretchfactor\fontdimen3\font minus
		\fontdimen4\font\relax}
	\providecommand\BIBforeignlanguage[2]{{%
			\expandafter\ifx\csname l@#1\endcsname\relax
			\typeout{** WARNING: IEEEtran.bst: No hyphenation pattern has been}%
			\typeout{** loaded for the language `#1'. Using the pattern for}%
			\typeout{** the default language instead.}%
			\else
			\language=\csname l@#1\endcsname
			\fi
			#2}}
	
	\bibitem{Janson2018}
	L.~Janson, T.~Hu, and M.~Pavone, ``Safe motion planning in unknown
	environments: Optimality benchmarks and tractable policies,'' in
	\emph{Proceedings of Robotics: Science and Systems}, Pittsburgh, PA, June
	26--30 2018.
	
	\bibitem{Dijkstra1959}
	E.~W. Dijkstra, ``A note on two problems in connexion with graphs,''
	\emph{Numerische Mathematik}, vol.~1, no.~1, pp. 269--271, 1959.
	
	\bibitem{Hart1968}
	P.~E. {Hart}, N.~J. {Nilsson}, and B.~{Raphael}, ``A formal basis for the
	heuristic determination of minimum cost paths,'' \emph{IEEE Transactions on
		Systems Science and Cybernetics}, vol.~4, no.~2, pp. 100--107, July 1968.
	
	\bibitem{Bohlin2000}
	R.~{Bohlin} and L.~E. {Kavraki}, ``Path planning using lazy {PRM},'' in
	\emph{IEEE International Conference on Robotics and Automation}, vol.~1, San
	Francisco, CA, April 24--28 2000, pp. 521--528.
	
	\bibitem{Karaman2011}
	S.~Karaman and E.~Frazzoli, ``Sampling-based algorithms for optimal motion
	planning,'' \emph{The International Journal of Robotics Research}, vol.~30,
	no.~7, pp. 846--894, 2011.
	
	\bibitem{Arslan2013}
	O.~Arslan and P.~Tsiotras, ``Use of relaxation methods in sampling-based
	algorithms for optimal motion planning,'' in \emph{IEEE International
		Conference on Robotics and Automation}, Karlsr\"{u}he, Germany, May 6--10
	2013, pp. 2421--2428.
	
	\bibitem{Janson2015}
	L.~Janson, E.~Schmerling, A.~Clark, and M.~Pavone, ``Fast marching tree: A fast
	marching sampling-based method for optimal motion planning in many
	dimensions,'' \emph{The International Journal of Robotics Research}, vol.~34,
	no.~7, pp. 883--921, 2015.
	
	\bibitem{Gammell2015}
	J.~D. {Gammell}, S.~S. {Srinivasa}, and T.~D. {Barfoot}, ``Batch informed trees
	({BIT}*): Sampling-based optimal planning via the heuristically guided search
	of implicit random geometric graphs,'' in \emph{IEEE International Conference
		on Robotics and Automation}, Seattle, WA, May 26--30 2015, pp. 3067--3074.
	
	\bibitem{Dellin2016}
	C.~M. Dellin and S.~S. Srinivasa, ``A unifying formalism for shortest path
	problems with expensive edge evaluations via lazy best-first search over
	paths with edge selectors,'' in \emph{Proceedings of the International
		Conference on Automated Planning and Scheduling}, no.~9, London, UK, 2016,
	pp. 459--467.
	
	\bibitem{Mandalika2019}
	A.~Mandalika, S.~Choudhury, O.~Salzman, and S.~Srinivasa, ``Generalized lazy
	search for robot motion planning: Interleaving search and edge evaluation via
	event-based toggles,'' in \emph{Proceedings of the International Conference
		on Automated Planning and Scheduling}, vol.~29, no.~1, Berkeley, CA, 2019,
	pp. 745--753.
	
	\bibitem{Strub2020}
	M.~P. Strub and J.~D. Gammell, ``Advanced {BIT}* ({ABIT}*): Sampling-based
	planning with advanced graph-search techniques,'' in \emph{IEEE International
		Conference on Robotics and Automation}, Paris, France, May 31--Aug 31 2020,
	pp. 130--136.
	
	\bibitem{Koenig2002}
	S.~Koenig and M.~Likhachev, ``D* lite,'' in \emph{Eighteenth National
		Conference on Artificial Intelligence}, Edmonton, Canada, July 28--Aug 1
	2002, p. 476–483.
	
	\bibitem{Koenig2004}
	S.~Koenig, M.~Likhachev, and D.~Furcy, ``Lifelong planning {A}*,''
	\emph{Artificial Intelligence}, vol. 155, no.~1, pp. 93 -- 146, 2004.
	
	\bibitem{Likhachev2008}
	M.~Likhachev, D.~Ferguson, G.~Gordon, A.~Stentz, and S.~Thrun, ``Anytime search
	in dynamic graphs,'' \emph{Artificial Intelligence}, vol. 172, no.~14, pp.
	1613 -- 1643, 2008.
	
	\bibitem{Aine2016}
	S.~Aine and M.~Likhachev, ``Truncated incremental search,'' \emph{Artificial
		Intelligence}, vol. 234, pp. 49 -- 77, 2016.
	
	\bibitem{Kambhampati1986}
	S.~Kambhampati and L.~S. Davis., ``Multiresolution path planning for mobile
	robots,'' \emph{IEEE Journal of Robotics and Automation}, pp. 135--145, 1986.
	
	\bibitem{Jung2007}
	D.~Jung, ``Hierarchical path planning and control of a small fixed-wing uav:
	theory and experimental validation.'' Ph.D. dissertation, Georgia Institute
	of Technology, Atlanta, GA, 2007.
	
	\bibitem{Cowlagi2012}
	R.~V. {Cowlagi} and P.~{Tsiotras}, ``Multiresolution motion planning for
	autonomous agents via wavelet-based cell decompositions,'' \emph{IEEE
		Transactions on Systems, Man, and Cybernetics, Part B}, vol.~42, no.~5, pp.
	1455--1469, Oct 2012.
	
	\bibitem{Hauer2015}
	F.~{Hauer}, A.~{Kundu}, J.~M. {Rehg}, and P.~{Tsiotras}, ``Multi-scale
	perception and path planning on probabilistic obstacle maps,'' in \emph{IEEE
		International Conference on Robotics and Automation}, Seattle, WA, May 26--30
	2015, pp. 4210--4215.
	
	\bibitem{Aliabdi2019}
	A.~AliAbdi, A.~Mohades, and M.~Davoodi, ``Constrained shortest path problems in
	bi-colored graphs: a label-setting approach,'' \emph{GeoInformatica}, pp. 1
	-- 19, 2019.
	
	\bibitem{Wooden2006}
	D.~Wooden and M.~Egerstedt, ``On finding globally optimal paths through
	weighted colored graphs,'' in \emph{Proceedings of the 45th IEEE Conference
		on Decision and Control}, San Diego, CA, December 13--15 2006, pp.
	1948--1953.
	
	\bibitem{Torreno2014}
	A.~Torre\~no, E.~Onaindia, and O.~Sapena, ``{FMAP}: Distributed cooperative
	multi-agent planning,'' \emph{Applied Intelligence}, vol.~41, no.~2, pp.
	606--626, 2014.
	
	\bibitem{Nissim2014}
	R.~Nissim and R.~Brafman, ``Distributed heuristic forward search for
	multi-agent planning,'' \emph{Journal of Artificial Intelligence Research},
	vol.~51, no.~1, pp. 293--332, 2014.
	
	\bibitem{Lim2020}
	J.~{Lim} and P.~{Tsiotras}, ``{MAMS-A}*: Multi-agent multi-scale {A}*,'' in
	\emph{IEEE International Conference on Robotics and Automation}, Paris,
	France, May 31--Aug 31 2020, pp. 5583--5589.
	
	\bibitem{Cohen2014}
	B.~{Cohen}, M.~{Phillips}, and M.~{Likhachev}, ``Planning single-arm
	manipulations with n-arm robots,'' in \emph{Proceedings of Robotics: Science
		and Systems}, Berkeley, CA, July 12--16 2014.
	
	\bibitem{Zelinsky1992}
	A.~{Zelinsky}, ``A mobile robot exploration algorithm,'' \emph{IEEE
		Transactions on Robotics and Automation}, vol.~8, no.~6, pp. 707--717, Dec
	1992.
	
	\bibitem{Choset1994}
	H.~{Choset} and J.~{Burdick}, ``Sensor based planning and nonsmooth analysis,''
	in \emph{IEEE International Conference on Robotics and Automation}, San
	Diego, CA, May 8--13 1994, pp. 3034--3041.
	
\end{thebibliography}
\end{document}